\documentclass[a4paper,10pt]{article}

\usepackage[utf8]{inputenc} 
\usepackage[T1]{fontenc}    
\usepackage{hyperref}       
\usepackage{url}            
\usepackage{booktabs}       
\usepackage{amsfonts}       
\usepackage{microtype}      
\usepackage{amsmath,amssymb,latexsym,amsthm}
\usepackage{graphicx,subfigure}
\usepackage{tikz}
\usepackage[a4paper,margin=3cm]{geometry}
\usepackage{authblk}

\allowdisplaybreaks

\DeclareMathOperator{\sgn}{sgn}
\DeclareMathOperator{\tr}{Tr}
\DeclareMathOperator{\stw}{K}

\newtheorem{theorem}{Theorem}
\newtheorem{lemma}[theorem]{Lemma}

\newtheorem{definition}[theorem]{Definition}

\title{Similarity Learning for Time Series Classification}
\date{}
\begin{document}

\author[1,2]{Maria-Irina Nicolae}
\author[2]{Éric Gaussier}
\author[1]{Amaury Habrard}
\author[1]{Marc Sebban}
\affil[1]{Univ Lyon, UJM-Saint-Etienne, CNRS, Institut d Optique Graduate School, Laboratoire Hubert Curien UMR 5516, France}
\affil[2]{Grenoble-Alps University, CNRS, Laboratoire d'Informatique de Grenoble UMR 5317, France}

\maketitle

\begin{abstract}
  Multivariate time series naturally exist in many fields, like energy, bioinformatics, signal processing, and finance.
  Most of these applications need to be able to compare these structured data.
  In this context, dynamic time warping (DTW) is probably the most common comparison measure.
  However, not much research effort has been put into improving it by learning.
  In this paper, we propose a novel method for learning similarities based on DTW, in order to improve time series classification.
  Making use of the uniform stability framework, we provide the first theoretical guarantees in the form of a generalization bound for linear classification.
  The experimental study shows that the proposed approach is efficient, while yielding sparse classifiers.
\end{abstract}

\section{Introduction}
\label{sec:intro}
The presence of time series in numerous fields of application makes them the object of considerable research effort for their classification or prediction.
Classification for time series represents a challenging problem, with multiple applications in fields like speech recognition, energy consumption, object identification, bioinformatics, patient care, etc.
To solve such tasks, one is inherently brought to compare time series by pairs, in order to determine their closeness or common patterns.
However, time series coming from real applications are most of the time not directly comparable, because of the differences in length, phase or sampling frequency.
An important subsequent task for solving the previous problems becomes finding the right alignment between time moments.

Dynamic time warping~\cite{kruskall83} is the most well-known algorithm for measuring the similarity between two time series by finding the best alignment between them.
Its popularity is due to its capacity to work with series of varying lengths and phases, and its performance, usually much better than that of the Euclidean distance.
The majority of previous results in time series classification concerns the adjustment of the constraints for finding the best alignment between time series for the task at hand~\cite{Itakura75,Sakoe1978,Jeong2011,sparsedtw,fastdtw}.
Most of these approaches are designed for univariate time series~\cite{Keogh2003,journals/adac/PrekopcsakL12}, which record the value of only one feature per time moment.
When dealing with multivariate time series, one way of using these methods is to weigh features equally, but that does not take into account the semantics of the features, nor the possible difference in scale.
Metric learning~\cite{Bellet2015d,weinberger09distance,Davis:2007:IML:1273496.1273523,Xing2002} can address exactly this problem as it allows one to learn the weights of features and the correlations between them from the available training data.
This field is well developed for feature vectors, but the results concerning time series are scarce, mostly because of the complexity of the data.
Moreover, metric learning methods based for time series do not come with any theoretical insurance of improving classification results once the learned metric is plugged into a machine learning algorithm.


In this paper, we address this double limitation by learning similarities for time series with generalization guarantees.
Our method is based on the $(\epsilon, \gamma, \tau)$-good similarities learning framework~\cite{conf/colt/BalcanBS08}.
The learned similarity function is used to induce a linear separator with good classification guarantees in the feature space.
We prove that our method has uniform stability, which allows us to derive a generalization bound for both the learned metric and the classifier.
To our knowledge, this is the first approach to provide theoretical guarantees for time series classification.
We prove the efficiency of our method through an experimental study on UCI datasets~\cite{Lichman:2013}.

The rest of the paper is organized as follows.
Section~\ref{sec:related} presents a brief overview of state of the art in metric learning and time series.
Section~\ref{sec:main} introduces the proposed similarity learning approach, while Section~\ref{sec:generalize} is dedicated to theoretical results.
In Section~\ref{sec:results}, we present an experimental study comparing our method to the state of the art.

\section{Related Work}
\label{sec:related}

In this section, we give an overview of some background knowledge on DTW, time series classification, metric learning and the $(\epsilon, \gamma, \tau)$-good framework. For the rest of this paper, we shall refer to scalars in regular font ($\gamma$), vectors in bold lower case ($\mathbf{x}$) and matrices in bold upper case ($\mathbf{M}$).

\begin{figure}[t!]
\centering
\resizebox{.6\linewidth}{!}{%
\begin{tikzpicture}
\draw[step=.4cm,gray,thin] (0,0) grid (4,3.2);
\draw[thick,-] (0,0) -- (4,0) node[anchor=north west] {$t_{\mathbf{A}}$};
\draw[thick,-] (0,0) -- (0,3.2) node[anchor=south east] {$t_{\mathbf{B}}$};
\foreach \x in {1,2,3,4,5,6,7,8,9,10}
    \draw (\x * 0.4cm - 0.2cm,1pt) -- (\x * 0.4cm - 0.2cm,-1pt) node[anchor=north] {$\x$};
\foreach \y in {1,2,3,4,5,6,7,8}
    \draw (1pt,\y * 0.4cm - 0.2cm) -- (-1pt,\y * 0.4cm - 0.2cm) node[anchor=east] {$\y$};
\draw[fill] (0.2,0.2) circle [radius=.08]
  -- (0.6,0.2) circle [radius=.08]
  -- (1,0.6) circle [radius=.08]
  -- (1.4,1) circle [radius=.08]
  -- (1.8,1.4) circle [radius=.08]
  -- (2.2,1.4) circle [radius=.08]
  -- (2.6,1.4) circle [radius=.08]
  -- (2.6,1.8) circle [radius=.08]
  -- (2.6,2.2) circle [radius=.08]
  -- (3,2.6) circle [radius=.08]
  -- (3.4,3) circle [radius=.08]
  -- (3.8,3) circle [radius=.08];
\draw[fill] (6,1.5) circle [radius=.08] node[anchor=north] {Allowed moves};
\draw[->] (6,1.5) -- (6.4,1.5);
\draw[->] (6,1.5) -- (6.4,1.9);
\draw[->] (6,1.5) -- (6,1.9);
\end{tikzpicture}
}
\caption{Example of optimal path of length 12 found by DTW.}
\label{fig:dtw}
\end{figure}
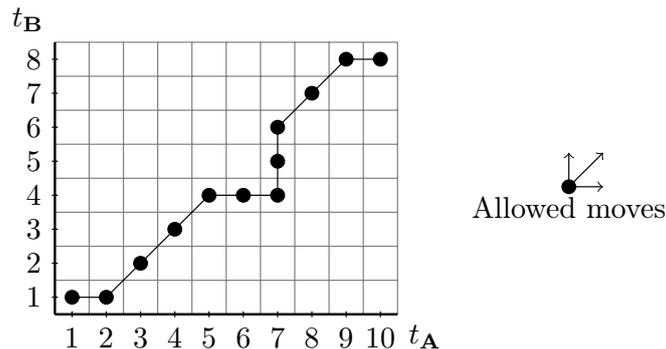

\textbf{Dynamic time warping}~\cite{kruskall83} computes the optimal alignment between two time series under a metric by finding the pairs of time indices to align.
This method was designed to solve the problem of comparing time series of different lengths and phases.
The warping path found by DTW is computed w.r.t.\ a cost matrix (typically constructed with the Euclidean distance) in quadratic time through dynamic programming.
Aligning two time series means finding all the matching time moments between them (Figure~\ref{fig:dtw}).
The alignment is well constructed if all indices in both time series are used and the warping path is continuous and monotonically increasing.
This implies that the first and last points are respectively aligned. 
To respect monotonicity, for each step of the alignment $(i,j)$ there are only three subsequent moves possible: $(i+1,j), (i,j+1)$ or $(i+1, j+1)$.

In order to overcome the computational complexity of DTW, faster alternatives were introduced, like FastDTW~\cite{fastdtw} and SparseDTW~\cite{sparsedtw}.
Many variants were considered to constrain the global warping path, in order to speed up the algorithm and avoid pathological warping (\emph{e.g.} aligning the beginning of a series with the end of another), of which we mention the Sakoe-Chiba band~\cite{Sakoe1978} and the Itakura parallelogram~\cite{Itakura75}.

\textbf{Metric learning}~\cite{DBLP:journals/ftml/Kulis13, Bellet2015d} focuses on learning the parameters of a distance or similarity function from data.
The learned metric is then used to solve the task at hand in the same way as with standard metrics.
The most well-known parameterized distance used in metric learning is the Mahalanobis distance, defined for a pair of vectors $\mathbf{x}$ and $\mathbf{y}$ as $d_\mathbf{M}(\mathbf{x}, \mathbf{y})=\sqrt{(\mathbf{x}-\mathbf{y})^T\mathbf{M}(\mathbf{x}-\mathbf{y})}$.
$\mathbf{M}$ is the square positive semi-definite (PSD) matrix parameterizing the distance, whose entries we wish to learn from data.
Notice that when $\mathbf{M}$ is the identity matrix, the metric becomes the standard Euclidean distance.
The characteristics of the data are usually modeled as constraints from side information when learning the metric.
More exactly, the two main approaches are pair-based constraints (two points are similar or dissimilar) and triplet-based constraints (a given point is more similar to one point than to the other).
For (semi-)supervised tasks, it is straightforward to generate the constraints from class information.

Metric learning for feature vectors has received important attention over the past years.
Most of the methods are designed for nearest neighbor classification.
Large Margin Metric Learning (LMNN)~\cite{weinberger09distance} and Information-Theoretic Metric Learning (ITML)~\cite{Davis:2007:IML:1273496.1273523} are probably the most well-known methods for feature vectors.
They learn a Mahalanobis distance from triplet, respectively pair constraints by enforcing an intuitive geometric criterion: bringing the points from the same class together, while pushing those from other classes away.
ITML introduces for the first time LogDet divergence regularization, used later in several other distance learning methods.

For time series, the notion of learning a metric has mostly been used in the sense of learning the right alignment for univariate time series~\cite{conf/ida/FrambourgCG13}.
To our knowledge, learning a metric as a transformation in the features of time series has only been explored by few methods. In~\cite{NIPS2014_5383}, the authors propose to learn a Mahalanobis metric for multivariate time series alignment of audio data. One significant limitation of their approach is that they consider the true alignment a priori known for their audio problem, information that is not available in most of the cases.
Recently, LDMLT~\cite{maha_dtw_2015} was designed to learn a Mahalanobis distance for multivariate time series from triplet constraints.
The method does so with an iterative approach that minimizes the loss of the triplets under LogDet regularization, to ensure the metric stays PSD.
Experiments are performed for nearest neighbor and SVM classification.
However, the loss function they use for the metric learning step is not related to the losses of the classifiers using it afterward.
Moreover, neither the method from~\cite{NIPS2014_5383} nor LDMLT come with guarantees that learning the metric improves performance for the given task.

\paragraph*{"Good" similarity functions}
The $(\epsilon, \gamma,\tau)$-good framework is one of the first to relate the characteristics of a similarity function based on non necessarily PSD matrices to its performance in classification.
For this, they define the notion of "goodness" for a similarity function.
Consider a binary classification setting over labeled examples $(\mathbf{x}, l)$ coming from a distribution $\mathcal{P}$ over $\mathcal{X}\times\{+1, -1\}$. The hinge loss is defined as $[1-c]_+=\max(0,1-c)$.
\begin{definition} \cite{conf/colt/BalcanBS08} \label{def:hinge}
$K:\mathcal{X} \times \mathcal{X} \rightarrow [-1,1]$ is a $(\epsilon, \gamma,\tau)$-good similarity function in hinge loss for a learning problem P if there exists a random indicator function $R(\mathbf{x})$ defining a probabilistic set of "reasonable points" such that the following conditions hold: \\
  1. We have
  $ \mathbb{E}_{(\mathbf{x}, l)\sim \mathcal{P}}\left[\left[1-lg(\mathbf{x})/\gamma\right]_+\right]\leq \epsilon, $
  where $g(\mathbf{x}) = \mathbb{E}_{(\mathbf{x}',l'),R(\mathbf{x}')}\left[l'K(\mathbf{x}, \mathbf{x}')|R(\mathbf{x}')\right]$. \\
  2. $\Pr_{\mathbf{x}'}(R(\mathbf{x}')) \geq \tau$.
\end{definition}
This definition is based on a set of reasonable points, that are used to create the feature space.
In practice, these points are obtained by drawing from $\mathcal{P}$ an (unlabeled) sample $\mathcal{L}=\{\mathbf{x}'_1, \mathbf{x}'_2,\ldots,\mathbf{x}'_{d_u}\}$ of $d_u$ random "landmarks".
The first condition of the definition imposes that an $(1-\epsilon)$ proportion of examples $\mathbf{x}$ should be on average $2\gamma$ more similar to reasonable examples $\mathbf{x}'$ of their own label than to random reasonable examples of the other label.
The margin violations are averaged over all reasonable points, which is easier to satisfy than pair- or triplet-based constraints, as required by LMNN or ITML.
The second condition sets the minimum number of reasonable points to a proportion of $\tau$.
In this definition, nothing is said about the form of the similarity function, so it is generic.
Definition~\ref{def:hinge} can be used to learn a linear separator from an $(\epsilon, \gamma,\tau)$-good similarity:
\begin{theorem}\cite{conf/colt/BalcanBS08}\label{theo:hinge}
Let $K$ be an $(\epsilon, \gamma, \tau)$-good similarity function in hinge loss for a learning problem $\mathcal{P}$.
For any $\epsilon_1>0$ and $0<\delta<\gamma\epsilon_1/4$ let $\mathcal{L}$ be a sample of $d_u=\frac{2}{\tau}\left ( log(2/\delta)+16\frac{log(2/\delta)}{(\epsilon_1\gamma)^2} \right )$ (unlabeled) landmarks drawn from $\mathcal{P}$.
Consider the mapping $\phi^\mathcal{L}:\mathcal{X}\rightarrow \mathbb{R}^{d_u}$, $\phi^\mathcal{L}_i(\mathbf{x})=K(\mathbf{x}, \mathbf{x}'_i), i\in \{1,\ldots, d_u\}$.
With probability $1-\delta$ over the random sample $\mathcal{L}$, the induced distribution $\phi^\mathcal{L}(\mathcal{P})$ in $\mathbb{R}^{d_u}$, has a separator achieving hinge loss at most $\epsilon+\epsilon_1$ at margin $\gamma$.
\end{theorem}

In other words, if $K$ is $(\epsilon, \gamma, \tau)$-good according to Definition~\ref{def:hinge} and enough data is available, there exists a linear separator $\boldsymbol{\alpha}$ with error arbitrarily close to $\epsilon$ in the space $\phi^\mathcal{L}$.
Given a labeled learning sample of size $d_l$, the separator is found by solving the following linear program:
\begin{equation}
\min_{\boldsymbol{\alpha}} \Big\{ \sum_{i=1}^{d_l} \Big[1-\sum_{j=1}^{d_u} \alpha_j l_i K(\mathbf{x}_i, \mathbf{x}_j)\Big]_+  : \sum_{j=1}^{d_u} |\alpha_j|\leq 1/\gamma \Big\}. \label{eq:linear-problem}
\end{equation}
As the problem is $L_1$-constrained, tuning the value of $\gamma$ may produce a sparse solution.
This formulation is equivalent to a relaxed L1-norm SVM~\cite{zhu20041}.
Lastly, the associated classifier takes the following form:
\begin{equation}
y = \sgn \sum_{j=1}^{d_u} \alpha_j K(\mathbf{x}, \mathbf{x}_j). \label{eq:predictor}
\end{equation}

The main limitation of this approach is however that the similarity function $K$ is supposed known, and they do not provide a way to design such similarities. This issue has been addressed in~\cite{Bellet2012a} only for feature vectors.
The objective of this paper is to provide a solution in a more complex setting of time series.

\section{Similarity Learning for Time Series Classification}
\label{sec:main}

This section presents the proposed method for learning temporal similarity functions.
We start by defining the similarity to be used with time series, then present the method for learning it.
Let $\mathbf{A} \in \mathbb{R}^{t_{\mathbf{A}} \times d}$ be a multivariate time series of length $t_{\mathbf{A}}$ and dimension $d$.
We denote by $\mathcal{X}$ the space of all time series of finite length.
Now consider the following binary classification problem: we are given labeled multivariate time series $(\mathbf{A}, l)$  drawn from a distribution $\mathcal{P}$ over $\mathcal{X} \times \{+1, -1\}$, possibly of different lengths, but of same dimension $d$.

\subsection{Bilinear Similarity for Time Series}

For a pair of time series $\mathbf{A}$ and $\mathbf{B}$, let $\mathbf{C}_{\mathbf{M}}(\mathbf{A}, \mathbf{B})\in \mathbb{R}^{t_{\mathbf{A}} \times {t_{\mathbf{B}}}}$ be a pairwise matrix of the cost of aligning a time moment in $\mathbf{A}$ to one in $\mathbf{B}$ under the metric parameterized by the matrix $\mathbf{M}$.
As we use a similarity function, $\mathbf{C}_{\mathbf{M}}(\mathbf{A}, \mathbf{B})$ represents the affinity scores that we want to maximize instead of the cost to be minimized.
We refer to the rows of $\mathbf{A}$ as $\mathbf{a}_1,\ldots,\mathbf{a}_{t_{\mathbf{A}}}$ and those of $\mathbf{B}$ as $\mathbf{b}_1,\ldots,\mathbf{b}_{t_{\mathbf{B}}}$.
Without loss of generality, the data is normalized as $||\mathbf{a}_i||_2 = 1, i \in \{1 \ldots t_{\mathbf{A}}\}, \forall \mathbf{A} \in \mathcal{X}$.
We will focus on an affinity matrix of form:
\[
  \mathbf{C}_{\mathbf{M}}(\mathbf{A}, \mathbf{B})_{i,j} = \mathbf{a}_i^T \cdot \mathbf{M} \cdot \mathbf{b}_j,
\]
where $\mathbf{M}$ is the matrix parameterizing the metric.
For the pair of indices $i$ and $j$, the affinity is equivalent to computing the generalized cosine similarity~\cite{ICDM09_Qamar}, as $\mathbf{a}_i$ and $\mathbf{b}_j$ are already normalized.
The same operation can be written using only matrices:
\[
  \mathbf{C}_{\mathbf{M}}(\mathbf{A}, \mathbf{B}) = \mathbf{A} \cdot \mathbf{M} \cdot \mathbf{B}^T.
\]

$\mathbf{C}_{\mathbf{M}}$ can be used to compute the alignment between two time series with DTW.
Given this affinity matrix, let $\mathbf{Y}\in\{0,1\}^{t_{\mathbf{A}}\times t_{\mathbf{B}}}$ be a binary matrix encoding an alignment between $\mathbf{A}$ and $\mathbf{B}$: $\mathbf{Y}_{\mathbf{A}, \mathbf{B}}^{ij} = 1$ if the time moment $i$ from $\mathbf{A}$ is aligned with the moment $j$ from $\mathbf{B}$ and zero otherwise. The length of the alignment is noted $t_{\mathbf{AB}}$.
Computing the score of aligning $\mathbf{A}$ and $\mathbf{B}$ from the affinity matrix and the alignment can be written as the following similarity function:
\begin{align*}  \stw_{\mathbf{M}}(\mathbf{A}, \mathbf{B}) & = \tr(\mathbf{C}_\mathbf{M}(\mathbf{A}, \mathbf{B})^T \cdot \mathbf{Y}_{\mathbf{AB}}) / t_{\mathbf{AB}} \\
  & = \tr(\mathbf{B} \cdot \mathbf{M}^T \cdot \mathbf{A}^T \cdot \mathbf{Y}_{\mathbf{AB}}) / t_{\mathbf{AB}}.
\end{align*}

When computing the product between the affinity matrix and the alignment, the scores of the pairs of points that are aligned end up on the main diagonal of the resulting matrix.
Applying the trace operator sums only these diagonal values, while discarding the others.
As the value of the similarity is cumulative, we normalize it w.r.t.\ the length of the alignment in order to remove the bias created by very long alignments.
Using $\stw_{\mathbf{M}}$ as similarity function to compare multivariate time series allows us to take advantage of the ideal alignment, while considering an advantageous weighting of the features and cross-features for each time moment.
An important property is that the metric matrix $\mathbf{M}$ does not have to be PSD.
We shall now discuss a method for learning $\mathbf{M}$ from data.

\subsection{Learning Good Similarities}
Our objective is to learn the matrix $\mathbf{M}$ that parameterizes the similarity function $\stw_{\mathbf{M}}$ for usage in classification.
For this, we dispose of a training set $\mathcal{S}$ of $m$ time series $\{(\mathbf{A}_i, l_i)\}_{i=1}^{m}$ drawn accordingly to $\mathcal{P}$
and a set $\mathcal{L}$ of $n$ landmarks $\{(\mathbf{B}_j, l_j')\}_{j=1}^{n}$ from the same distribution.
We want to optimize the $(\epsilon, \gamma, \tau)$-goodness of the proposed similarity function as presented in Definition~\ref{def:hinge}:
\[
  \mathbb{E}_{(\mathbf{A}, l)}\left[\left[1-\mathbb{E}_{(\mathbf{B}, l'),R(\mathbf{B})}\left[ll'\stw_\mathbf{M}(\mathbf{A}, \mathbf{B}))|R(\mathbf{B})\right]/\gamma\right]_+\right]\leq \epsilon.
\]
As this criterion is defined over true expected values, we shall improve its empirical version instead.
When optimizing the goodness criterion, we do so w.r.t.\ the set of landmarks $\mathcal{L}$.
We assume for now that they are fixed.
Notice that two heuristics for choosing them from data are discussed in the supplementary material.
Learning the similarity w.r.t.\ Definition~\ref{def:hinge} is equivalent to learning the entries of the matrix $\mathbf{M}$ that parameterizes it and is done by solving the following optimization problem over $\mathbf{M}$:
\begin{equation} \label{eq:obj}
  \min_{\mathbf{M}} \frac{1}{m}\sum_{(\mathbf{A}, l) \in \mathcal{S}} \left[1 - \frac{1}{n \gamma} \sum_{j=1}^{n} l l'_j \stw_{\mathbf{M}}(\mathbf{A}, \mathbf{B}_j)\right]_+ + \lambda ||\mathbf{M}||^2_{\mathcal{F}}.
\end{equation}

Notice that the similarity function $\stw_{\mathbf{M}}$ is linear in $\mathbf{M}$.
Problem~\eqref{eq:obj} is thus convex and can easily be solved.
In order to avoid overfitting, the objective function is regularized with the squared Frobenius norm of the matrix $\mathbf{M}$.
Using this regularizer will allow us to provide theoretical guarantees for the proposed approach through uniform stability.
Tuning the regularization parameter $\lambda$ controls the tradeoff between fitting the data and limiting the complexity of the hypothesis.
We call the proposed method Similarity Learning for Time Series (SLTS).
After solving Problem~\eqref{eq:obj}, $\stw_{\mathbf{M}}$ is plugged in Equation~\eqref{eq:linear-problem} in order to learn the linear separator $\boldsymbol{\alpha}$.
Having a formulation based on landmarks implies that the value of the similarity function (and indirectly of the alignment using DTW) only needs to be computed for the data points w.r.t.\ the set of landmarks.
As computing $\stw_{\mathbf{M}}$ is expensive, the lower the number of landmarks, the faster the computation.

\section{Theoretical Guarantees}
\label{sec:generalize}
Learning the metric by solving Problem~\eqref{eq:obj} places our approach in the $(\epsilon, \gamma, \tau)$ framework, which enforces the theoretical guarantees from Theorem~\ref{theo:hinge} for the learned classifier.
In this section, we derive a generalization bound for SLTS using the notion of uniform stability~\cite{Bousquet:2002:SG:944790.944801}.
This bound provides a link between the empirical loss we are minimizing under regularization in Equation~\eqref{eq:obj} and the value we want to minimize, the true loss.
We start by making the following notations.
Let the empirical loss function for an example $(\mathbf{A}, l)\sim\mathcal{P}$ be
\[
  \ell(\mathbf{M}, (\mathbf{A}, l)) = \left[1 - \frac{1}{n} \sum_{j=1}^{n} l l'_j \stw_{\mathbf{M}}(\mathbf{A}, \mathbf{B}_j)/\gamma \right]_+.
\]

According to Problem~\eqref{eq:obj}, SLTS minimizes the empirical risk of the learned matrix $\mathbf{M}$ over the whole training set $\cal{S}$:
\[
  \hat{\mathcal{E}}_{\mathcal{S}}(\mathbf{M}) = \frac{1}{m} \sum_{(\mathbf{A}, l) \in \cal{S}} \ell(\mathbf{M}, (\mathbf{A}, l)).
\]

According to Definition~\ref{def:hinge}, the error that the algorithm should minimize is the true expectation:
\[
\mathcal{E}_{\mathcal{P}}(\mathbf{M}) = \mathbb{E}_{(\mathbf{A},l) \sim \mathcal{P}} \left[ \ell (\mathbf{M}, (\mathbf{A}, l))\right].
\]

We will denote by $\mathcal{S}^i$ the training set obtained from $\mathcal{S}$ by replacing the $i$th example with a new one coming from the same distribution.
We now define the uniform stability of an algorithm.
\begin{definition}[Uniform stability~\cite{Bousquet:2002:SG:944790.944801}]
A learning algorithm has a uniform stability in $\frac{\kappa}{m}$, with $\kappa \geq 0$ constant, if $\forall i$,
\[
  \sup_{(\mathbf{A}, l)\sim \mathcal{P}} |\ell(\mathbf{M}, (\mathbf{A}, l)) - \ell(\mathbf{M}^i, (\mathbf{A}, l))| \leq \frac{\kappa}{m},
\]
where $\mathbf{M}$ is the metric learned on the training set $\mathcal{S}$, and $\mathbf{M}^i$ is the metric learned on $\mathcal{S}^i$.
\end{definition}

Uniform stability ensures a certain robustness of the learned metric w.r.t.\ small variations in the training set.
This property enables us to derive a generalization bound on the true error of an algorithm.
To prove the stability of SLTS, we first need to show that the considered loss function is bounded and $k$-lipschitz: the smaller $k$, the more stable the algorithm.
We do so in Lemmas~\ref{lem:bound_loss} and~\ref{lem:lips}.
All the proofs of the Lemmas in this section are presented in the supplementary material.
\begin{lemma}[Bound on the loss function] \label{lem:bound_loss}
Let $(\mathbf{A}, l)$ be an example and $\mathbf{M}$ the minimizer of Problem~\eqref{eq:obj}. Then
$$\ell(\mathbf{M}, (\mathbf{A}, l)) \leq \frac{\sqrt{2d}}{\gamma \sqrt{\lambda}}.$$
\end{lemma}

\begin{lemma}[$k$-lipschitz continuity] \label{lem:lips}
Let $\mathbf{M}$ and $\mathbf{M}'$ be two matrices and $(\mathbf{A}, l)$ an example.
The loss function $\ell$ is $k$-lipschitz with $k = \frac{\sqrt{2d}}{\gamma}$ such that:
\[
  |\ell(\mathbf{M}, (\mathbf{A}, l)) - \ell(\mathbf{M}', (\mathbf{A}, l))| \leq k ||\mathbf{M} - \mathbf{M}'||_{\mathcal{F}}.
\]
\end{lemma}

The property of $k$-lipschitzness implies that the loss variation is proportional to the difference between $\mathbf{M}$ and $\mathbf{M}'$.
We can now prove that our approach has uniform stability.
\begin{lemma} \label{lem:uni}
Given a training sample $\mathcal{S}$ of $m$ examples drawn i.i.d.\ from $\mathcal{P}$, our algorithm SLTS has uniform stability in $\kappa/m$ with $\kappa = \frac{4d}{\gamma^2\lambda}$.
\end{lemma}

Having now shown the uniform stability of SLTS, we are ready to derive the generalization bound.
For this, Lemmas~\ref{lem:exp_risk} and~\ref{lem:risk_bound} are necessary, providing bounds on quantities that intervene in the proof of the bound.
Let $\mathcal{R}_{\mathcal{S}} = \mathcal{E}_{\mathcal{P}}(\mathbf{M}) - \hat{\mathcal{E}}_{\mathcal{S}}(\mathbf{M})$. We need to bound the quantities $\mathbb{E}_{\mathcal{S}}[\mathcal{R}_{\mathcal{S}}]$ and $|\mathcal{R}_{\mathcal{S}} - \mathcal{R}_{\mathcal{S}^i}|$.
\begin{lemma} \label{lem:exp_risk}
For a learning method of estimation error $\mathcal{R}_{\mathcal{S}}$ and satisfying a uniform stability of $\kappa/m$, we have:
\[
  \mathbb{E}_{\mathcal{S}} [\mathcal{R}_{\mathcal{S}}] \leq \frac{\kappa}{m}.
\]
\end{lemma}
\begin{lemma} \label{lem:risk_bound}
For any metric $\mathbf{M}$ learned by solving Problem~\eqref{eq:obj} on a training set $\mathcal{S}$ of $m$ samples, and a loss function $\ell$ bounded according to Lemma~\ref{lem:bound_loss}, we have:
\[
  |\mathcal{R}_{\mathcal{S}} - \mathcal{R}_{\mathcal{S}^i}| \leq \frac{2\kappa}{m} + \frac{\sqrt{2d}}{m\gamma \sqrt{\lambda}}.
\]
\end{lemma}
\begin{theorem}[Generalization bound] \label{th:gen_bound}
With probability $1-\delta$, for any matrix $\mathbf{M}$ learned by solving Problem~\eqref{eq:obj}, we have:
\[
\mathcal{E}_{\mathcal{P}}(\mathbf{M}) \leq \mathcal{E}_{\mathcal{S}}(\mathbf{M}) + \frac{4d}{\gamma^2 \lambda m} + \left( \frac{4d}{\gamma^2 \lambda} + \frac{1}{\gamma} \sqrt{\frac{2d}{\lambda}} \right) \sqrt{\frac{2 \log{\frac{2}{\delta}}}{m}}.
\]
\end{theorem}
\begin{proof}
Using McDiarmid's inequality and Lemma~\ref{lem:risk_bound}, we can write:
\begin{equation}
  \Pr [\mathcal{R}_{\mathcal{S}} - \mathbb{E}[\mathcal{R}_{\mathcal{S}}] \geq \epsilon] \leq 2 \exp \left( -\frac{2\epsilon^2}{m \left(\frac{2 \kappa + p}{m}\right)^2} \right). \label{eq:applied_mcdiar}
\end{equation}
By setting $\delta = 2 \exp \left( -\frac{2\epsilon^2}{m \left(\frac{2 \kappa + p}{m}\right)^2} \right)$ in Inequality~\eqref{eq:applied_mcdiar}, we obtain:
\[
  \epsilon = \sqrt{\frac{2}{m} \left( \frac{4d}{\gamma^2 \lambda} + \frac{1}{\gamma}\sqrt{\frac{2d}{\lambda}}\right)^2 \log \frac{2}{\delta}}.
\]
Then, with probability $1-\delta,
  \mathcal{R}_{\mathcal{S}} = \mathcal{E}(\mathbf{M}) - \hat{\mathcal{E}}_{\mathcal{S}}(\mathbf{M}) < \mathbb{E}[\mathcal{R}_{\mathcal{S}}] + \epsilon \iff
  \mathcal{E}_{\mathcal{P}}(\mathbf{M}) < \hat{\mathcal{E}}_{\mathcal{S}}(\mathbf{M}) + \frac{\kappa}{m} + \epsilon
$.
Replacing the values of $\kappa$ and $\epsilon$ in the previous inequality yields the bound.
\end{proof}

The result from Theorem~\ref{th:gen_bound} shows the consistency of the proposed similarity learning approach.
The bound converges with a standard rate of $1/\sqrt{m}$ in the number of samples.
According to~\cite{conf/nips/VermaB15}, the presence of the number of features $d$ in the numerator of the bound is to be expected and shows that the approach may suffer from the curse of dimensionality.
High values of $d$ can be compensated by increasing either the size of $\mathcal{S}$, or the value of the regularization parameter $\lambda$, present in the denominator.
SLTS minimizes the empirical error of the $(\epsilon, \gamma, \tau)$ framework, thus reducing the error rate $\epsilon$.
By plugging the metric learned by SLTS into the framework, we obtain a guarantee on the performance of the associated classifier.

\section{Experiments}
\label{sec:results}
In this section, we present the results of the experiments conducted to evaluate the performance of the proposed method.
In the first experiment, we show that learning the matrix $\mathbf{M}$ brings additional information for linear classification.
We also analyze the influence of the number of landmarks on SLTS.
The second study provides a comparison of SLTS to the state-of-the-art algorithms while the third part illustrates the capacity of SLTS to learn a discriminant metric in the feature space created by the landmarks.
Finally, we provide a discussion over the choice of landmarks, followed by an additional experiment meant to compare a few of heuristics for landmarks selection.
We conduct the experimental study on multivariate time series datasets coming from UCI Machine Learning Repository~\cite{Lichman:2013}, containing between 47-8800 instances.
We start by giving the description of the datasets used for the experiments in Table~\ref{tab:datasets}.
In the case of Auslan, we only use the 25 first classes instead of the total of 95, as done in precedent studies~\cite{maha_dtw_2015}.
The dataset Robot execution failure contains five subtasks (LP 1-5), that are treated separately.

\begin{table}[t]
\centering
\begin{small}
\caption{Properties of the datasets used in the experimental study.}
\label{tab:datasets}
\medskip
\begin{tabular}{lcccc}
\toprule
Dataset & \#Instances & Length & \#Feat. & \#Classes \\
\midrule
Japanese vowels & 640 & 7-29 & 12 & 9 \\
Auslan & 675 & 47-95 & 22 & 25 \\
Arabic digits & 8800 & 4-93 & 13 & 10 \\
Robot exec. failure \\
~~~LP1 & 88 & 15 & 6 & 4 \\
~~~LP2 & 47 & 15 & 6 & 5 \\
~~~LP3 & 47 & 15 & 6 & 4 \\
~~~LP4 & 117 & 15 & 6 & 3 \\
~~~LP5 & 164 & 15 & 6 & 5 \\
\bottomrule
\end{tabular}
\end{small}
\end{table}

We compare our method against the following classic algorithms:
\begin{itemize}
  \item Standard nearest neighbor classifier (1NN);
  \item Linear SVM under $L_2$ regularization;
  \item Linear classifier from~\cite{conf/colt/BalcanBS08}, presented in Equation~\eqref{eq:predictor} (called BBS from now on);
  \item LDMLT~\cite{maha_dtw_2015} with a nearest neighbor classifier;
  \item SLTS, the similarity learning method proposed in this chapter, which is then used to learn a global linear classifier using the formulation in \cite{conf/colt/BalcanBS08}.
\end{itemize}

To propose a fair comparative study, all the methods that do not learn a metric use the proposed bilinear form as similarity function (with $\mathbf{M}$ set to the identity matrix) computed with the DTW alignment on the scalar product.
As confirmed by the experiment presented in the supplementary material, landmarks are randomly chosen for BBS and SLTS.
We use all the classifiers in their binary version, in a one-vs-rest setting.
We recall here that each time moment is normalized to ensure the $L_2$ norm equals 1.
For this experimental study, we have access to a standard training/test partitioning for Japanese vowels and Arabic digits datasets, while Robot execution failure (LP1-LP5) and Auslan are randomly split to 70\% training/30\% test data.
For all datasets, we retain 30\% of the training set for hyperparameter tuning.
We perform experiments on 10 different splits and present the average result with a 95\% confidence interval.
Cross-validation is performed to tune the following parameters: $C\in\{2^{-6},\ldots,2^9\}$ for SVM, $\gamma \in\{10^{-4},\ldots,10^1\}$ for BBS, both when used separately or joint to SLTS, and $\lambda \in\{0.1,1,10\}$ for SLTS.

\begin{figure}[h]
  \centering
  \subfigure[Auslan]{\label{fig:arabic}\includegraphics[width=.4\textwidth]{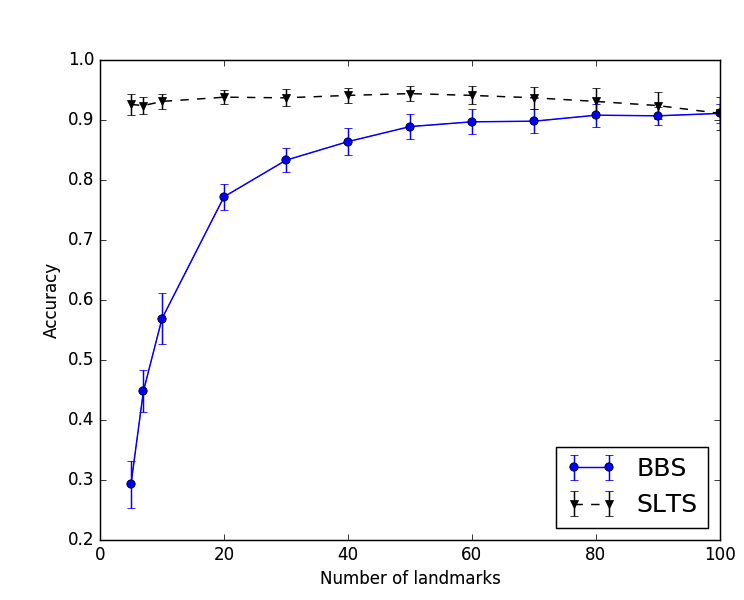}}
  \hfill
  \subfigure[Arabic digits]{\label{fig:auslan}\includegraphics[width=.4\textwidth]{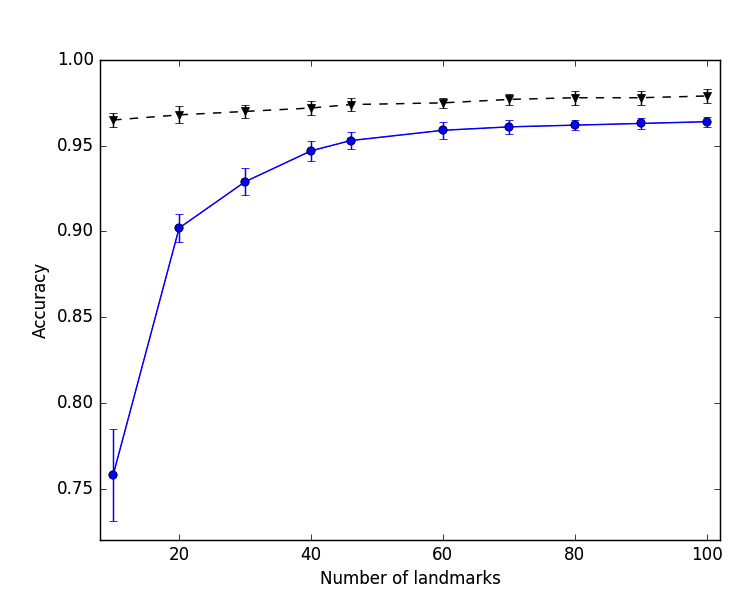}}
  \\
  \subfigure[Japanese vowels]{\label{fig:jap}\includegraphics[width=.4\textwidth]{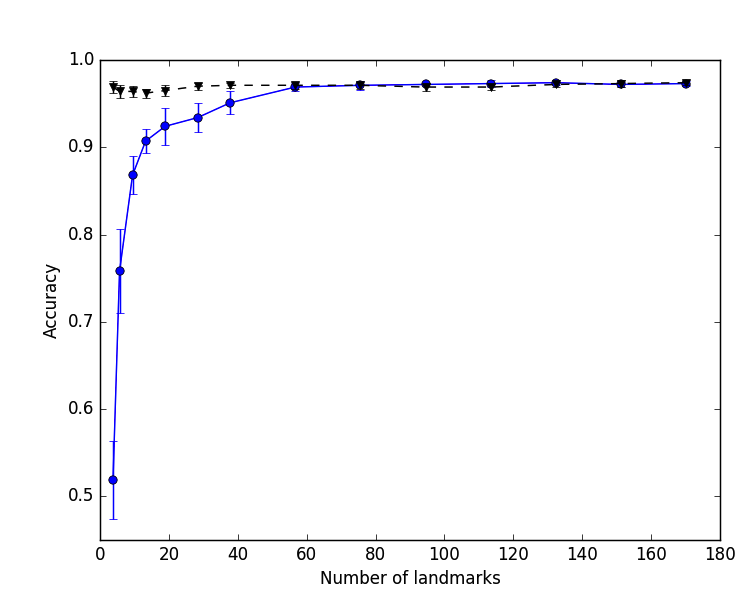}}
  \hfill
  \subfigure[LP1]{\label{fig:lp1}\includegraphics[width=.4\textwidth]{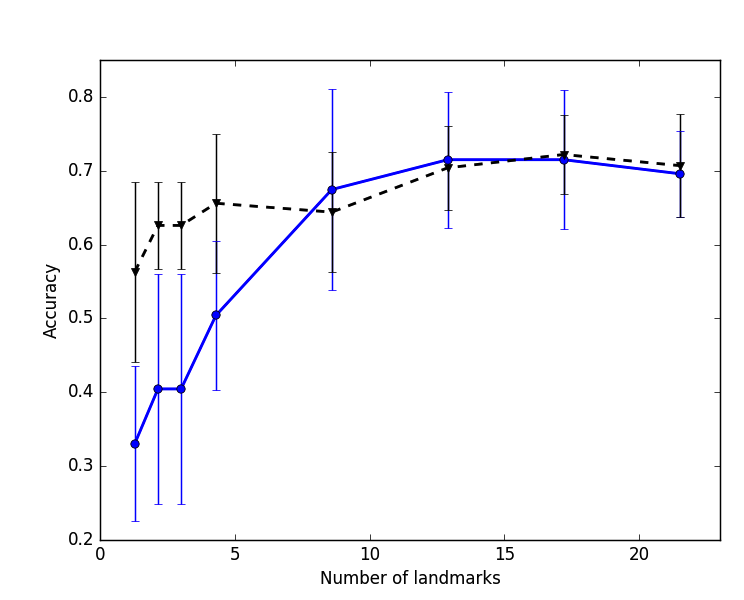}}
  \\
  \subfigure[LP2]{\label{fig:lp2}\includegraphics[width=.4\textwidth]{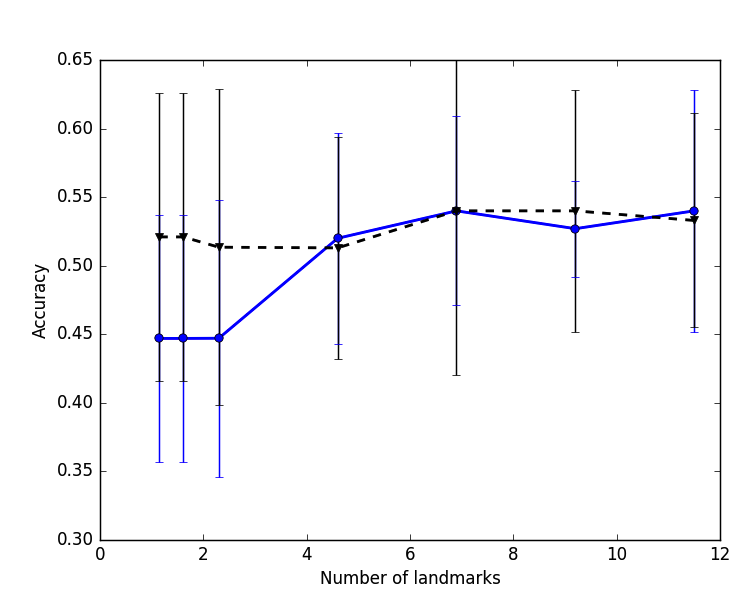}}
  \hfill
  \subfigure[LP3]{\label{fig:lp3}\includegraphics[width=.4\textwidth]{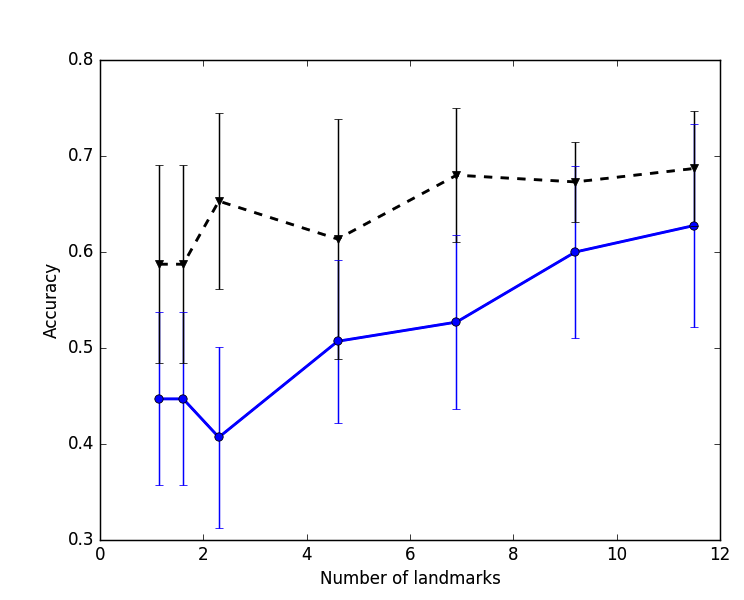}}
  \\
  \subfigure[LP4]{\label{fig:lp4}\includegraphics[width=.4\textwidth]{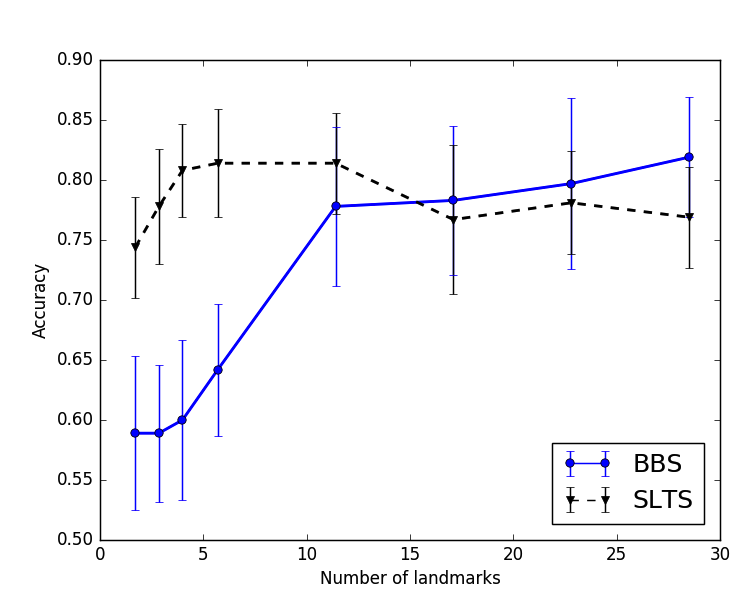}}
  \hfill
  \subfigure[LP5]{\label{fig:lp5}\includegraphics[width=.4\textwidth]{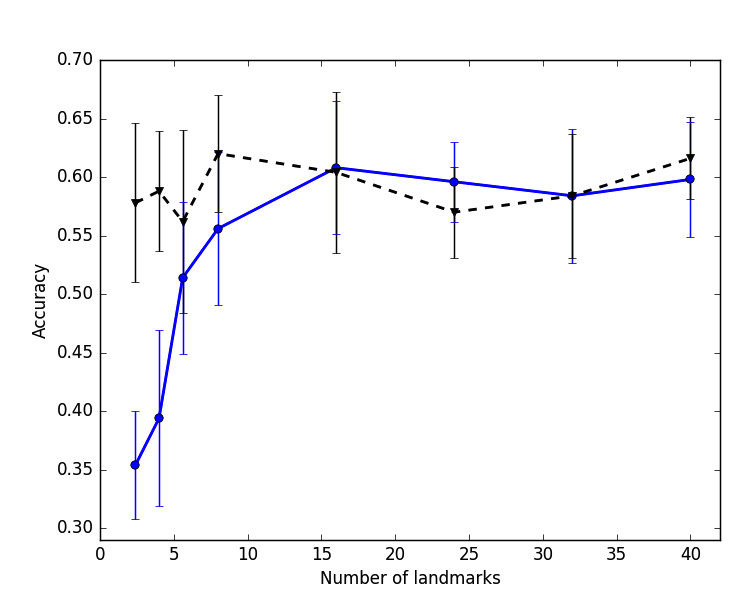}}
  \caption[Classification accuracy of BBS and SLTS on time series w.r.t.\ the number of landmarks]{Classification accuracy of BBS and SLTS w.r.t.\ the number of landmarks.}
  \label{fig:compare}
\end{figure}

\begin{table}
  \centering
  \begin{small}
  \caption{Classification accuracy (\%) with confidence interval at 95\%.}
  \label{tab:acc}
  \medskip
  \begin{tabular}{llllll}
  \toprule
  Method & Japanese vowels & Auslan & Arabic digits & Robot exec. failure & Avg.\\
  \midrule
  1NN & 93.8 & 77.8$\pm$2.1 & 94.7 & 68.8$\pm$7.5 & 92.1 \\
  LDMLT & 97.3 & 95.0$\pm$1.3 & 96.9 & 71.9$\pm$7.0 & 95.6 \\
  SVM & 97.8$\pm$0.1 & 92.6$\pm$0.1 & 93.3$\pm$0.0 & 60.6$\pm$6.5 & 92.2 \\
  BBS & 97.1$\pm$0.5 & 91.1$\pm$1.6 & 96.4$\pm$0.3 & 66.9$\pm$10.6 & 94.7 \\
  SLTS & 97.1$\pm$0.4 & 91.1$\pm$2.7 & 97.9$\pm$0.4 & 67.0$\pm$7.8 & 95.8 \\
  \bottomrule
  \end{tabular}
  \end{small}
\end{table}

\paragraph*{Behavior of SLTS and impact of the number of landmarks}
We (i)~show that SLTS improves linear classification compared to BBS and (ii)~analyze the influence of the quantity of landmarks on the accuracy obtained for BBS and SLTS.
We consider the range of up to 50\% of the size of the training set as landmarks for small datasets, or up to 100 landmarks for the others.
The results of this study are presented in Figure~\ref{fig:compare}.
The accuracy of SLTS is almost always higher than that of BBS, showing the improvement that can be obtained through similarity learning.
When a reasonable quantity of data is available (Figures~\ref{fig:arabic}-\ref{fig:jap}), SLTS achieves a performance close to its best value even with a few landmarks, thus performing well even with a low quantity of data.
Overall, BBS has difficulties providing a good classifier based on a small number of landmarks, but the results of the method significantly improve with more landmarks.
We explain the high variability of the results of BBS and SLTS on LP1-LP5 by the small sizes of the tasks.

\paragraph*{Classification performance comparison}
The results of the comparison of SLTS and BBS with other methods are displayed in Table~\ref{tab:acc} (both SLTS and BBS are based on the maximum number of landmarks from the previous experiment). No confidence interval in the table values means that the train/test split of the data is already provided, and the output of the method is deterministic. As one can note, among global methods relying on a linear classification ({\it i.e.}, SLTS, BBS, and SVM), both SLTS  and BBS perform better than SVM (they are on a par on Japanese vowels, slightly below on Auslan, and above on Arabic digits and Robot exec. failure). Using a Student $t$-test for paired samples on the average reveals that SLTS is significantly better than BBS and SVM. This shows the usefulness of the $(\epsilon, \gamma, \tau)$-good framework as well as the importance of metric learning in this framework. The comparison of SLTS with local methods (as 1NN and LDMLT) yields more contrasted results. On all datasets except Robot exec. failure, 1NN is significantly below SLTS according to a Student $t$-test. However, compared to LDMLT, SLTS is on a par on Japanese vowels, below on Auslan and Robot exec. failure, and above on Arabic digits (a Student $t$-test on the average does not reveal any significant difference between the two methods). LDMLT relies on both a local method and a metric learned, which suggests again that learning a metric is beneficial on these datasets. This said, LDMLT learns a distance, whereas all the other methods rely on a similarity. The comparison between the two should thus be taken with caution as distances and similarities can yield very different results \cite{ICDM09_Qamar}.

\paragraph*{Visualization of the similarity space}
To illustrate the transformation induced in the feature space by learning the metric, we propose a visualization experiment on the Japanese vowels dataset using 10 landmarks chosen randomly.
We compute the value of the similarity function $\stw_{\mathbf{M}}$ for all the data w.r.t.\ the landmarks, first without metric learning ($\mathbf{M}=I$), then with the metric learned for each of the 9 classes.
In all the cases, we apply PCA to the values of the similarity function and plot the first two components.
We thus obtain a 2D representation of the feature space, of which we present in Figure~\ref{fig:pca} the case of the initial feature space and that of the metric learned for the first three classes.
In similarity space with no metric learning (Figure~\ref{fig:jap1}), all the data points are mixed, independently of their label.
In Figures~\ref{fig:jap2}-\ref{fig:jap4}, each metric linearly separates the class it has learned to discriminate from the others.
For the learned similarities, the first two components of PCA explain around 98\% of the variance, while with no similarity learning this value is around 86\%.
This study proves that learning an $(\epsilon, \gamma, \tau)$-good similarity function changes the representation space towards better class discrimination, making it suitable for learning a large margin linear separator.

\begin{figure}[h!]
  \centering
  \subfigure[No metric learning]{\label{fig:jap1}\includegraphics[width=.44\textwidth]{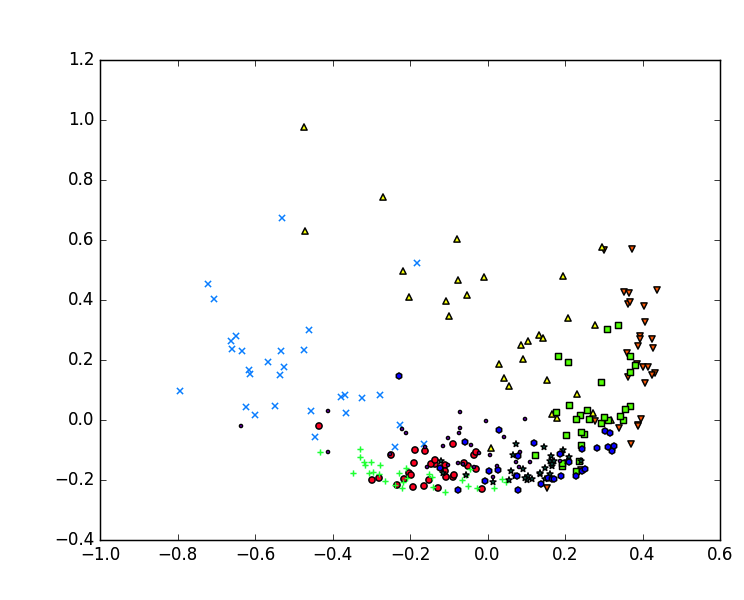}}
  \subfigure[Metric for class 1]{\label{fig:jap2}\includegraphics[width=.44\textwidth]{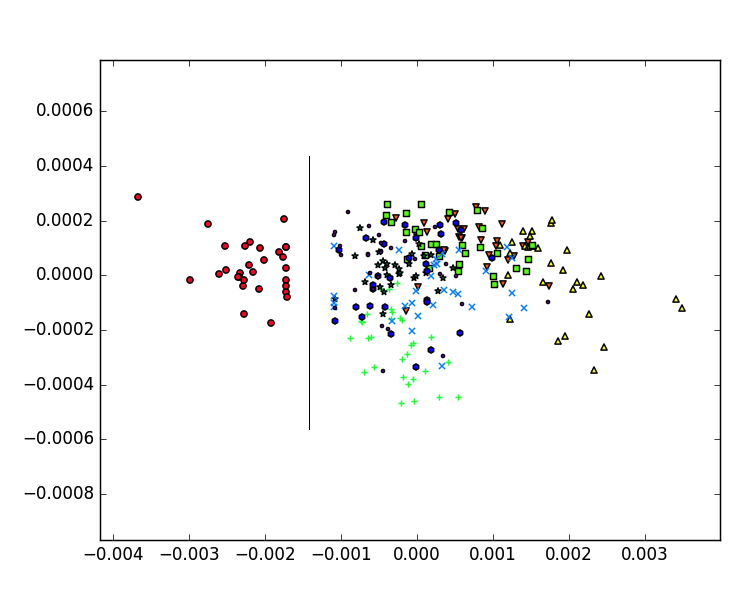}}
  \\
  \subfigure[Metric for class 2]{\label{fig:jap3}\includegraphics[width=.44\textwidth]{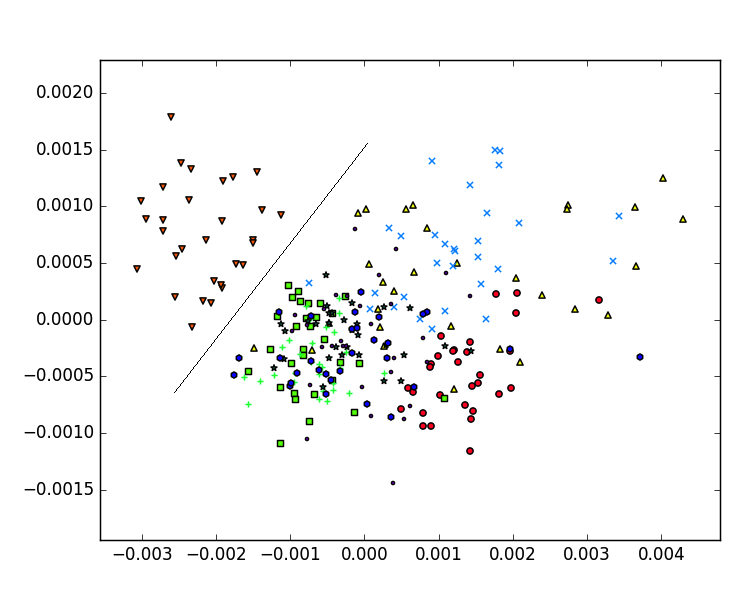}}
  \subfigure[Metric for class 3]{\label{fig:jap4}\includegraphics[width=.44\textwidth]{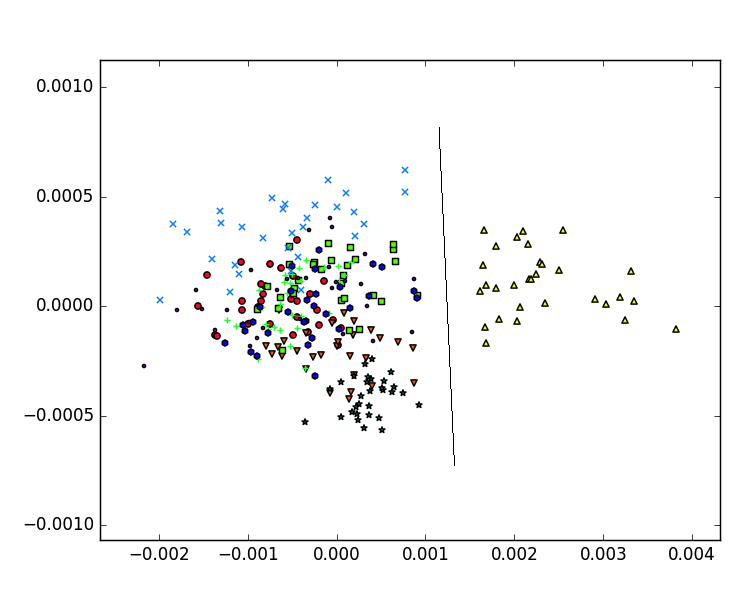}}
  \caption{PCA (first two components) in the similarity space on Japanese Vowels. Each class becomes linearly separable from the others when using its corresponding metric.}
  \label{fig:pca}
\end{figure}

\paragraph*{Heuristics for choosing the landmarks}
We have previously assumed we have access to a set of landmarks for the construction of the feature space.
We will now discuss two heuristics for choosing the most representative points in the training set as landmarks, before presenting experimental results concerning the performance of each of these methods.
\textbf{K-Medoids}~\cite{kaufman87} is a classical clustering technique. The resulting medoids representing the clusters are points of the initial dataset, that will be subsequently used as landmarks. 
\textbf{Dselect}~\cite{NIPS2011_4306} was proposed as a landmarks selection algorithm that optimizes a criterion of diversity. Starting with a randomly chosen landmark, at each iteration the algorithm greedily adds to the set of landmarks the training point that is least similar to the ones already selected.
Note that for both selection heuristics the number of landmarks needs to be set in advance. Also, none of these methods exploits the information from the labels of the time series.
In the case where no prior information is available for the classification task, the set of landmarks can also be selected randomly from the training set, with the risk of relying upon non informative landmarks.

\begin{table}[t]
  \parbox{.48\linewidth}{
  \begin{small}
  \caption{Classification accuracy for landmarks selection methods on Japanese vowels.}
  \label{tab:jap_lmks}
  \centering
  \tabcolsep=0.2cm
  \begin{tabular}{lccc}
    \toprule
    Lmks & DSelect & KMedoids & Random \\
    \midrule
    3\% & \textbf{96.8}$\pm$0.6 & 96.3$\pm$0.6 & 96.4$\pm$0.7 \\
    5\% & \textbf{96.5}$\pm$0.6 & 95.9$\pm$0.7 & 96.4$\pm$0.6 \\
    7\% & \textbf{97.0}$\pm$0.2 & 96.3$\pm$0.4 & 96.2$\pm$0.5 \\
    10\% & \textbf{97.3}$\pm$0.3 & 96.2$\pm$0.5 & 96.5$\pm$0.6 \\
    15\% & \textbf{97.1}$\pm$0.2 & 96.4$\pm$0.4 & 97.0$\pm$0.4 \\
    20\% & \textbf{97.1}$\pm$0.3 & 96.8$\pm$0.4 & \textbf{97.1}$\pm$0.4 \\
    30\% & 96.7$\pm$0.4 & 97.0$\pm$0.3 & \textbf{97.1}$\pm$0.4 \\
    40\% & 97.0$\pm$0.4 & 96.9$\pm$0.3 & \textbf{97.1}$\pm$0.4 \\
    50\% & 96.8$\pm$0.3 & \textbf{96.9}$\pm$0.3 & \textbf{96.9}$\pm$0.4 \\
    \bottomrule
  \end{tabular}
  \end{small}
  }
  \hfill
  \parbox{.48\linewidth}{
  \begin{small}
  \caption{Classification accuracy for landmarks selection methods on LP1.}
  \label{tab:lp1_lmks}
  \centering
  \begin{tabular}{lrrr}
    \toprule
    Lmks & DSelect & KMedoids & Random \\
    \midrule
    3\% & 48.5$\pm$6.2 & 50.7$\pm$9.4 & \textbf{56.3}$\pm$12.2 \\
    5\% & \textbf{67.8}$\pm$9.7 & 63.3$\pm$10.1 & 62.6$\pm$5.9 \\
    7\% & \textbf{67.8}$\pm$9.7 & 63.3$\pm$10.1 & 62.6$\pm$5.9 \\
    10\% & \textbf{68.5}$\pm$9.6 & 65.2$\pm$5.5 & 65.6$\pm$9.4 \\
    15\% & 67.0$\pm$6.4 & \textbf{69.3}$\pm$6.8 & 68.5$\pm$10.0 \\
    20\% & \textbf{71.1}$\pm$5.6 & 68.9$\pm$5.8 & 64.4$\pm$8.1 \\
    30\% & \textbf{70.4}$\pm$7.2 & 66.7$\pm$6.8 & \textbf{70.4}$\pm$5.7 \\
    40\% & \textbf{74.4}$\pm$7.9 & 70.4$\pm$6.5 & 72.2$\pm$5.3 \\
    50\% & \textbf{73.0}$\pm$7.0 & 71.5$\pm$8.5 & 70.7$\pm$7.0 \\
    \bottomrule
  \end{tabular}
  \end{small}
  }
\end{table}

We now present in Tables~\ref{tab:jap_lmks} and~\ref{tab:lp1_lmks} the classification results after learning the similarity with SLTS on landmarks selected using the presented heuristics.
DSelect and KMedoids are compared against landmarks selected randomly as baseline, in order to determine if they are indeed informative.
We perform these experiments on two small datasets, Japanese vowels and LP1. 
The mass of chosen landmarks is selected as a percentage of the total size of the training set and goes up to 50\%.
For Japanese vowels (Table~\ref{tab:jap_lmks}), all three methods perform almost the same for all amounts of landmarks.
DSelect reaches its best performance when the selected points represent 10-20\% of the training set, while KMedoids works best around 30\% mass of landmarks.
Overall, DSelect and Random heuristics yield better performance than KMedoids.
The results using Random show that SLTS can learn well ever when no computational effort is put into choosing the landmarks.
In the case of LP1 (Table~\ref{tab:lp1_lmks}), and in contrast to the Japanese vowels dataset, the performance of all the heuristics improves when increasing the number of landmarks.
The best results are obtained for 40\% mass of landmarks in the case of DSelect and Random, and 50\% for KMedoids.
For this dataset, the results are less stable, inducing larger confidence intervals.
For this reason, even though the best accuracy is attained by DSelect, its improvement over Random is not necessarily significant.
KMedoids is this time also the least performant heuristic.

KMedoids and DSelect have by themselves a computational complexity that is not to be ignored when working on large datasets.
Even so, their main disadvantage for time series is not the algorithmic complexity in itself, but the necessary precomputations.
One needs to compute the value of the similarity function for all pairs of time series, including the alignment, in order to be able to apply these heuristics.
This limitation goes directly against the main advantage of working with methods based on landmarks, like SLTS.
In view of this aspect and the previous experimental results, we have only considered the Random heuristic when comparing SLTS against state of the art algorithms on bigger datasets.

\section{Conclusion and Perspectives}

In this paper, we address the problem of learning a global linear classifier for multivariate time series through similarity learning.
We propose a bilinear similarity function that takes into account the optimal alignment.
Our method comes with a generalization bound on the error of the metric and of the classifier and is the first to provide classification performance guarantees for the learned similarity in the case of time series.
The experimental study proves the usefulness of the $(\epsilon, \gamma, \tau)$-good framework as well as the importance of metric learning in this setting. 
Future work should include learning Mahalanobis metrics, as suggested by the results of LDMLT.
We also plan on trying to capture local temporal information by learning multiple metrics, as well as on studying the impact of different regularizers on the matrix $\mathbf{M}$.

\subsubsection*{Acknowledgments}
Funding for this project was provided by a grant from Région Rhône-Alpes.

\appendix
\setcounter{theorem}{0}
\renewcommand{\thetheorem}{\Alph{theorem}}
\section{Proofs of Lemmas}
This section contains the proofs of Lemmas~\ref{lem:bound_loss} to~\ref{lem:risk_bound}, as well as defining some additional lemmas necessary for these proofs.

To prove Lemma~\ref{lem:bound_loss} from the paper, we need two additional lemmas.
Lemma~\ref{lem:bound_m} bounds the Frobenius norm of the learned matrix $\mathbf{M}$, while Lemma~\ref{lem:frob_bound} puts a bound on the Frobenius norm of a subpart of the similarity function.
For the sake of clarity, lemmas introduced in this supplementary material are numbered with capital letters.

\begin{lemma} \label{lem:bound_m}
If $\mathbf{M}$ is the optimal solution of Problem~\eqref{eq:obj}, we have:
\[
  ||\mathbf{M}||_{\mathcal{F}} \leq \frac{1}{\sqrt{\lambda}}.
\]
\end{lemma}

\begin{proof}
Since $\mathbf{M}$ is the optimal solution of Problem~\eqref{eq:obj}, we have:
\begin{align}
R_{\mathcal{S}}(\mathbf{M}) &\leq R_{\mathcal{S}}(\mathbf{0}) \nonumber \\
\frac{1}{m} \sum_{(\mathbf{A}, l) \in \mathcal{S}} \ell (\mathbf{M}, (\mathbf{A}, l)) & + \lambda ||\mathbf{M}||_{\mathcal{F}}^2 \nonumber \\
\leq \frac{1}{m} & \sum_{(\mathbf{A}, l) \in \mathcal{S}} \ell (\mathbf{0}, (\mathbf{A}, l)) + \lambda ||\mathbf{0}||_{\mathcal{F}}^2 \nonumber \\
\lambda ||\mathbf{M}||_{\mathcal{F}}^2 &\leq \frac{1}{m} \sum_{(\mathbf{A}, l) \in \mathcal{S}} \ell (\mathbf{0}, (\mathbf{A}, l)) \label{eq:lower_bd} \\
\lambda ||\mathbf{M}||_{\mathcal{F}}^2 &\leq 1 \label{eq:loss_zero} \\
||\mathbf{M}||_{\mathcal{F}} &\leq \frac{1}{\sqrt{\lambda}} \nonumber
\end{align}

Inequality~\eqref{eq:lower_bd} is a result of the fact that the hinge loss is always positive, while Inequality~\eqref{eq:loss_zero} comes from noting that the loss is bounded by $1/m$ when the metric is set to zero.
\end{proof}

\begin{lemma}[Technical lemma] \label{lem:frob_bound}
Let $\mathbf{A}\in \mathbb{R}^{t_{\mathbf{A}}\times d}$ and $\mathbf{B}\in \mathbb{R}^{t_{\mathbf{B}}\times d}$ be two examples, and $\mathbf{Y}_{\mathbf{AB}}\in \{0,1\}^{t_{\mathbf{A}}\times t_{\mathbf{B}}}$ of length $t_{\mathbf{AB}}$. Then
  \[
    ||\mathbf{A}^T \cdot \mathbf{Y}_{\mathbf{AB}} \cdot \mathbf{B}||_{\mathcal{F}} \leq t_{\mathbf{AB}}\sqrt{2d}.
  \]
\end{lemma}

\begin{proof}
\begin{align*}
& ||\mathbf{A}^T \cdot \mathbf{Y}_{\mathbf{AB}} \cdot \mathbf{B}||_{\mathcal{F}} \\
&= \sqrt{\sum_{i=1}^{d}\sum_{t=1}^{d} \left( \sum_{j=1}^{t_{\mathbf{A}}} \sum_{k=1}^{t_{\mathbf{B}}} a_{ij} y_{jk} b_{kt} \right)^2}  \\
  &= \sqrt{\sum_{i=1}^{d}\sum_{t=1}^{d} \left( 2\sum_{j=1}^{t_{\mathbf{A}}} \sum_{k=1}^{t_{\mathbf{B}}} \sum_{j'=1}^{t_\mathbf{A}} \sum_{k'=1}^{t_\mathbf{B}}  (a_{ij} y_{jk} b_{kt})(a_{ij'} y_{j'k'} b_{k't}) - \sum_{jk} (a_{ij} y_{jk} b_{kt})^2 \right)} \\
  &\leq \sqrt{\sum_{i,t}2\sum_{j,k} \sum_{j',k'} |(a_{ij} y_{jk} b_{kt})(a_{ij'} y_{j'k'} b_{k't})|} \\
  & = \sqrt{2 \sum_{j,k}y_{jk}\sum_{j',k'}y_{j'k'} + \sum_{i,t}a_{ij}a_{ij'}b_{kt}b_{k't}} \\
  &\leq \sqrt{2 t_{\mathbf{AB}}^2 |\sum_{i} a_{ij} a_{ij'} \sum_{t} b_{kt}b_{k't}|} \\
  &= \sqrt{2 t_{\mathbf{AB}}^2 \max_{i} a_{ij} \sum_{i} |a_{ij}| \max_{t} b_{kt} \sum_{t} |b_{kt}|} \\
  & = \sqrt{2 t_{\mathbf{AB}}^2 \max_{i} a_{ij} ||\mathbf{a}_i||_1 \max_{t} b_{kt} ||\mathbf{b}_k||_1} \\
  &\leq \sqrt{2 t_{\mathbf{AB}}^2 \sqrt{d} ||\mathbf{a}_i||_2 \sqrt{d} ||\mathbf{b}_k||_2} \\
  &\leq \sqrt{2 t_{\mathbf{AB}}^2 \cdot d} \\
  &= t_{\mathbf{AB}}\sqrt{2d}.
\end{align*}
\end{proof}

We are now able to present the proof of Lemma~\ref{lem:bound_loss}:
\begin{proof}[Proof of Lemma~\ref{lem:bound_loss}]
\begin{align}
  \ell(\mathbf{M}, (\mathbf{A}, l)) &= \left[1 - \frac{1}{n} \sum_{(\mathbf{B}, l') \in \mathcal{L}} l l' \stw_{\mathbf{M}} (\mathbf{A}, \mathbf{B}) / \gamma \right]_+ \nonumber \\
  &\leq \left| l \frac{1}{n} \sum_{(\mathbf{B}, l') \in \mathcal{L}} l' \stw_{\mathbf{M}} (\mathbf{A}, \mathbf{B}) / \gamma \right| \label{eq:lossb_hinge} \\
  &\leq \frac{1}{n} \sum_{(\mathbf{B}, l') \in \mathcal{L}} \left| l' \stw_{\mathbf{M}} (\mathbf{A}, \mathbf{B}) / \gamma \right| \label{eq:lossb_tri} \\
  &\leq \frac{1}{n\gamma} \sum_{(\mathbf{B}, l') \in \mathcal{L}} \left| \tr(\mathbf{M}^T \mathbf{A}^T \mathbf{Y}_{\mathbf{AB}} \mathbf{B}) / t_{\mathbf{AB}} \right| \nonumber \\
  &\leq \frac{1}{n\gamma} \sum_{(\mathbf{B}, l') \in \mathcal{L}} \frac{1}{t_{\mathbf{AB}}} ||\mathbf{M}||_{\mathcal{F}} ||\mathbf{A}^T \mathbf{Y}_{\mathbf{AB}} \mathbf{B}||_{\mathcal{F}} \nonumber \\
  &\leq \frac{1}{n\gamma} \sum_{(\mathbf{B}, l') \in \mathcal{L}} \frac{1}{t_{\mathbf{AB}}} \frac{1}{\sqrt{\lambda}} t_{\mathbf{AB}}\sqrt{2d} \label{eq:lossb_values} \\
  &\leq \frac{\sqrt{2d}}{\gamma \sqrt{\lambda}}. \nonumber
\end{align}

Equation~\eqref{eq:lossb_hinge} comes from the 1-lipschitzness of the hinge loss.
Inequality~\eqref{eq:lossb_tri} is obtained by applying triangle inequality. We obtain line~\eqref{eq:lossb_values} by applying Lemmas~\ref{lem:bound_m} and~\ref{lem:frob_bound}.
\end{proof}

\begin{proof}[Proof of Lemma~\ref{lem:lips}]
\begin{align}
& |\ell(\mathbf{M}, (\mathbf{A}, l)) - \ell(\mathbf{M}', (\mathbf{A}, l))| \nonumber \\
= & \left| \left[1 - \frac{1}{n} \sum_{(\mathbf{B}, l') \in \mathcal{L}} ll' \stw_{\mathbf{M}}(\mathbf{A}, \mathbf{B}) / \gamma \right]_+ - \left[1 - \frac{1}{n} \sum_{(\mathbf{B}, l') \in \mathcal{L}} ll' \stw_{\mathbf{M}'}(\mathbf{A}, \mathbf{B}) / \gamma \right]_+ \right| \nonumber \\
\leq & \left| \frac{1}{n} \sum_{(\mathbf{B}, l') \in \mathcal{L}} ll' \stw_{\mathbf{M}}(\mathbf{A}, \mathbf{B}) / \gamma - \frac{1}{n} \sum_{(\mathbf{B}, l') \in \mathcal{L}} ll' \stw_{\mathbf{M}'}(\mathbf{A}, \mathbf{B}) / \gamma \right| \label{eq:lips_lips} \\
= & \frac{1}{n\gamma}\left| \sum_{(\mathbf{B}, l') \in \mathcal{L}} l' \left(\stw_{\mathbf{M}}(\mathbf{A}, \mathbf{B}) - \stw_{\mathbf{M}'}(\mathbf{A}, \mathbf{B}) \right) \right| \nonumber \\
\leq & \frac{1}{n\gamma}\sum_{(\mathbf{B}, l') \in \mathcal{L}} \left| \stw_{\mathbf{M}}(\mathbf{A}, \mathbf{B}) - \stw_{\mathbf{M}'}(\mathbf{A}, \mathbf{B}) \right| \label{eq:lips_tri} \\
= & \frac{1}{n\gamma}\sum_{(\mathbf{B}, l') \in \mathcal{L}} \left| \tr((\mathbf{M} - \mathbf{M}')^T \cdot \mathbf{A}^T \cdot \mathbf{Y}_{\mathbf{AB}} \cdot \mathbf{B}) / t_{\mathbf{AB}}\right| \nonumber \\
\leq & \frac{1}{n\gamma}\sum_{(\mathbf{B}, l') \in \mathcal{L}} \frac{1}{t_{\mathbf{AB}}} ||(\mathbf{M} - \mathbf{M}')^T \cdot \mathbf{A}^T \cdot \mathbf{Y}_{\mathbf{AB}} \cdot \mathbf{B})||_1 \\
\leq & \frac{1}{n\gamma} ||\mathbf{M} - \mathbf{M}'||_{\mathcal{F}} \sum_{(\mathbf{B}, l') \in \mathcal{L}} \frac{1}{t_{\mathbf{AB}}} ||\mathbf{A}^T \cdot \mathbf{Y}_{\mathbf{AB}} \cdot \mathbf{B}||_{\mathcal{F}} \label{eq:lips_norm} \\
\leq & \frac{\sqrt{2d}}{\gamma} ||\mathbf{M} - \mathbf{M}'||_{\mathcal{F}}.
\end{align}

Inequality~\eqref{eq:lips_lips} comes from the 1-lipschitzness of the hinge loss. Inequality~\eqref{eq:lips_tri} is obtained by applying triangle inequality.
By using Lemma~\ref{lem:frob_bound} on line~\eqref{eq:lips_norm}, we obtain the lemma.
\end{proof}

Recall the following notation for the objective function from Equation~\eqref{eq:obj}: 
$R_{\mathcal{S}}(\mathbf{M}) := \hat{\mathcal{E}}_{\mathcal{S}}(\mathbf{M}) + \lambda ||\mathbf{M}||^2_{\mathcal{F}}$.
The following lemma is used for the proof of the uniform stability of an algorithm.

\begin{lemma} \label{lem:metric_norms}
Let $R_{\mathcal{S}}(\cdot)$ and $R_{\mathcal{S}^i}(\cdot)$ be the functions to optimize, $\mathbf{M}$ and $\mathbf{M}^i$ their corresponding minimizers, and $\lambda$ the regularization parameter used.
Let $\Delta \mathbf{M} = \mathbf{M} - \mathbf{M}^i$.
Then we have, for $t\in [0, 1]$:
\[
  ||\mathbf{M}||^2_{\mathcal{F}} - ||\mathbf{M} - t\Delta \mathbf{M}||^2_{\mathcal{F}} + ||\mathbf{M}^i||^2_{\mathcal{F}} - ||\mathbf{M}^i + t\Delta \mathbf{M}||^2_{\mathcal{F}} \leq \frac{2kt}{\lambda m} ||\Delta \mathbf{M}||_{\mathcal{F}}.
\]
\end{lemma}

The proof is similar to the one of Lemma~20 in~\cite{Bousquet:2002:SG:944790.944801}, thus we shall omit it.
We use the previous lemma to prove the uniform stability of our approach.

\begin{proof}[Proof of Lemma~\ref{lem:uni}]
By setting $t=\frac{1}{2}$ in Lemma~\ref{lem:metric_norms}, we obtain after some computations:
\[
\frac{1}{2}||\Delta\mathbf{M}||_{\mathcal{F}}^2 \leq \frac{k}{\lambda m} ||\Delta \mathbf{M}||_{\mathcal{F}},
\]
which implies:
\[
||\Delta\mathbf{M}||_{\mathcal{F}} \leq \frac{2k}{\lambda m}.
\]

Since our loss is $k$-lipschitz, we have:
\[
|\ell(\mathbf{M}, (\mathbf{A}, l)) - \ell(\mathbf{M}^i, (\mathbf{A}, l))| \leq k||\Delta \mathbf{M}||_{\mathcal{F}} = \frac{2k^2}{\lambda m}
\]
For this loss function, $k=\frac{\sqrt{2d}}{\gamma}$, and setting $\kappa=\frac{4d}{\gamma^2 \lambda}$ proves the lemma.
\end{proof}

\begin{proof}[Proof of Lemma~\ref{lem:exp_risk}]
\begin{align}
\mathbb{E}_{\mathcal{S}} [\mathcal{R}_{\mathcal{S}}] & \leq \mathbb{E}_{\mathcal{S}} [\mathbb{E}_{(A, l)}[\ell(\mathbf{M}, (A, l))] - \hat{\mathcal{E}}_{\mathcal{S}} (\mathbf{M})]\nonumber \\
& \leq \mathbb{E}_{\mathcal{S}, (A, l)} \left[ \left| \ell(\mathbf{M}, (A, l)) - \frac{1}{m} \sum_{(A_i, l_i)\in \mathcal{S}} \ell(\mathbf{M}, (A_i, l_i)) \right| \right] \nonumber \\
& \leq \mathbb{E}_{\mathcal{S}, (A, l)} \left[ \left| \frac{1}{m} \sum_{(A_i, l_i)} \left( \ell(\mathbf{M}, (A, l)) - \ell(\mathbf{M}, (A_i, l_i)) \right) \right| \right] \nonumber \\
& \leq \mathbb{E}_{\mathcal{S}, (A, l)} \left[ \left| \frac{1}{m} \sum_{(A_i, l_i)} \left( \ell(\mathbf{M}^i, (A_i, l_i)) - \ell(\mathbf{M}, (A_i, l_i)) \right) \right| \right] \label{eq:change_samp} \\
& \leq \frac{\kappa}{m}. \label{eq:apply_uni}
\end{align}

Inequality~\eqref{eq:change_samp} comes from the fact that changing one point with another from the same distribution does not affect the expected value, while Inequality~\eqref{eq:apply_uni} results from applying triangle inequality and uniform stability (Lemma~\ref{lem:uni}).
\end{proof}

\begin{proof}[Proof of Lemma~\ref{lem:risk_bound}]
\begin{align}
& |\mathcal{R}_{\mathcal{S}} - \mathcal{R}_{\mathcal{S}^i}| \nonumber \\
= & |\mathcal{E}_{\mathcal{P}}(\mathbf{M}) - \hat{\mathcal{E}}_{\mathcal{S}}(\mathbf{M}) - \mathcal{E}_{\mathcal{P}}(\mathbf{M}^i) + \hat{\mathcal{E}}_{\mathcal{S}^i}(\mathbf{M}^i)| \nonumber \\
= & |\mathcal{E}_{\mathcal{P}}(\mathbf{M}) - \hat{\mathcal{E}}_{\mathcal{S}}(\mathbf{M}) - \mathcal{E}_{\mathcal{P}}(\mathbf{M}^i) + \hat{\mathcal{E}}_{\mathcal{S}^i}(\mathbf{M}^i) - \hat{\mathcal{E}}_{\mathcal{S}}(\mathbf{M}^i) + \hat{\mathcal{E}}_{\mathcal{S}}(\mathbf{M}^i) | \nonumber \\
\leq & |\mathcal{E}_{\mathcal{P}}(\mathbf{M}) - \mathcal{E}_{\mathcal{P}}(\mathbf{M}^i)| + |\hat{\mathcal{E}}_{\mathcal{S}}(\mathbf{M}^i) - \hat{\mathcal{E}}_{\mathcal{S}}(\mathbf{M})| + |\hat{\mathcal{E}}_{\mathcal{S}^i}(\mathbf{M}^i) - \hat{\mathcal{E}}_{\mathcal{S}}(\mathbf{M}^i)| \label{eq:tri1}\\
\leq & \mathbb{E}_{(\mathbf{A},l)}[|\ell(\mathbf{M}, (\mathbf{A}, l)) - \ell(\mathbf{M}^i, (\mathbf{A}, l))|] + |\hat{\mathcal{E}}_{\mathcal{S}}(\mathbf{M}^i) - \hat{\mathcal{E}}_{\mathcal{S}}(\mathbf{M})| + |\hat{\mathcal{E}}_{\mathcal{S}^i}(\mathbf{M}^i) - \hat{\mathcal{E}}_{\mathcal{S}}(\mathbf{M}^i)| \label{eq:tri2} \\
\leq & \frac{\kappa}{m} + |\hat{\mathcal{E}}_{\mathcal{S}}(\mathbf{M}^i) - \hat{\mathcal{E}}_{\mathcal{S}}(\mathbf{M})| + |\hat{\mathcal{E}}_{\mathcal{S}^i}(\mathbf{M}^i) - \hat{\mathcal{E}}_{\mathcal{S}}(\mathbf{M}^i)| \label{eq:uni1} \\
\leq & \frac{\kappa}{m} + \frac{1}{m} \sum_{(\mathbf{A},l) \in \mathcal{S}} \left|\ell(\mathbf{M}^i, (\mathbf{A}, l)) - \ell(\mathbf{M}, (\mathbf{A}, l))\right|  \nonumber \\
+ & |\hat{\mathcal{E}}_{\mathcal{S}^i}(\mathbf{M}^i) - \hat{\mathcal{E}}_{\mathcal{S}}(\mathbf{M}^i)| \nonumber \\
\leq & \frac{\kappa}{m} + \frac{\kappa}{m} + |\hat{\mathcal{E}}_{\mathcal{S}^i}(\mathbf{M}^i) - \hat{\mathcal{E}}_{\mathcal{S}}(\mathbf{M}^i)| \label{eq:uni2} \\
= & \frac{2\kappa}{m} + \frac{1}{m} |\ell(\mathbf{M}^i, (\mathbf{A}^i, l^i)) - \ell(\mathbf{M}^i, (\mathbf{A}, l))| \label{eq:samp_dif}\\
\leq & \frac{2\kappa}{m} + \frac{1}{m}|\ell(\mathbf{M}^i, (\mathbf{A}^i, l^i))| \label{eq:pos} \\
\leq & \frac{2\kappa}{m} + \frac{\sqrt{2d}}{m \gamma \sqrt{\lambda}} \label{eq:loss}
\end{align}
Inequalities~\eqref{eq:tri1} and~\eqref{eq:tri2} come from triangle inequality.
Inequalities~\eqref{eq:uni1} and~\eqref{eq:uni2} come from the uniform stability of our algorithm (Lemma~\ref{lem:uni}).
Line~\eqref{eq:samp_dif} comes from the fact that $\mathcal{S}$ and $\mathcal{S}^i$ differ only by example $i$.
We can write Inequality~\eqref{eq:pos} because the loss is always positive, and we get line~\eqref{eq:loss} by bounding the value of the loss function (Lemma~\ref{lem:bound_loss}).
\end{proof}

\small{\bibliographystyle{abbrv}}
\bibliography{references}
\end{document}